\documentclass{article}

\usepackage[preprint]{my_nips_2018}

\usepackage[utf8]{inputenc} 
\usepackage[T1]{fontenc}    
\usepackage{hyperref}       
\usepackage{url}            
\usepackage{booktabs}       
\usepackage{amsfonts}       
\usepackage{nicefrac}       
\usepackage{microtype}      

\usepackage{amsmath,amssymb,amsthm}
\usepackage[mathscr]{eucal}

\renewcommand{\cite}{\citep*}

\newcommand{\F}{\mathcal F} 
\newcommand{\Z}{\mathcal Z} 
\newcommand{\BorelZ}{\mathscr Z}
\newcommand{\Borel}{{\cal B}}
\renewcommand{\H}{\mathcal H} 
\renewcommand{\P}{\mathbb P} 
\newcommand{\E}{\mathbb E} 
\newcommand{\nats}{\mathbb{N}} 
\newcommand{\reals}{\mathbb{R}} 

\newcommand{\drift}{\Delta}
\newcommand{\change}{\Delta}
\newcommand{\disc}{\rho}

\DeclareSymbolFont{bbold}{U}{bbold}{m}{n}
\DeclareSymbolFontAlphabet{\mathbbold}{bbold}
\newcommand{\ind}{\mathbbold{1}}

\newcommand{\vc}{d}

\newcommand{\argmin}{\mathop{\rm argmin}}

\newcommand{\ignore}[1]{}
\newcommand{\oldstuff}[1]{}

\newsavebox{\savepar}

\newtheorem{theorem}{Theorem}

\newtheorem{lemma}{Lemma}

\title{Statistical Learning under Nonstationary \\Mixing Processes}

\author{
Steve Hanneke \\\texttt{steve.hanneke@gmail.com}
\And Liu Yang \\\texttt{liu.yang0900@outlook.com}}

\begin{document}

\maketitle

\begin{abstract}
  We study a special case of the problem of statistical
  learning without the i.i.d. assumption.  Specifically, we suppose a
  learning method is presented with a sequence of data points, and
  required to make a prediction (e.g., a classification) for each one,
  and can then observe the loss incurred by this prediction. We go
  beyond traditional analyses, which have focused on stationary mixing
  processes or nonstationary product processes, by combining these two
  relaxations to allow nonstationary mixing processes.  We are
  particularly interested in the case of $\beta$-mixing processes,
  with the sum of changes in marginal distributions growing sublinearly in the number of
  samples.  Under these conditions, we propose a learning method, and
  establish that for bounded VC subgraph classes, the cumulative
  excess risk grows sublinearly in the number of predictions, at a
  quantified rate.
\end{abstract}


\section{Introduction}

Our setting is that of stream-based prediction.  At each time $t$, we are given 
access to data points from times $1$ through $t-1$, and are required to produce
a predictor $f_{t}$, which is then evaluated on a new data point at time $t$.
We study this in the \emph{general learning setting} of \cite{vapnik:82,vapnik:98},
which represents the learning objective as an abstract optimization problem.
As an example, in the special case of classification, given access to pairs $(x_1,y_1),\ldots,(x_{t-1},y_{t-1})$,
we would be tasked with producing a function mapping an observed point $x_{t}$ to a classification $\hat{y}_{t}$,
and we would be evaluated on whether $\hat{y}_{t} \neq y_{t}$ (called a \emph{mistake}).
We are then interested in characterizing the rate of growth of the cumulative number of mistakes,
as we repeat this for increasing values of $t$.

To study this problem, we suppose the sequence of observations are
stochastic, subject to some restrictions on their distribution.
Several such restrictions are possible.  For instance, the most-common
assumption used in the vast majority of the statistical learning
literature is that the data are independent and identically
distributed (i.i.d.).
However, some efforts to relax this assumption have also been
explored.  There are essentially two main threads of work toward
relaxing this assumption: relaxing the independence assumption while
maintaining the assumption of identical distributions (or
stationarity), or relaxing the assumption of identical distributions
while maintaining the independence assumption.  In the present work,
we are interested in relaxing these assumptions jointly.
Before getting into the details, let us first briefly review these two threads of the
literature.

Most of the literature on relaxations of the independence assumption focuses on 
\emph{stationary mixing} processes.  At the extreme of this branch, the work of
\cite{adams:10} reveals that any VC class admits
a uniform law of large numbers under stationary ergodic processes.
In particular, this implies that the method of \emph{empirical risk minimization}
approaches excess risk zero in the limit.
However, one cannot establish \emph{rates} of convergence under such 
general conditions as ergodicity.  To establish such rates, other
works have therefore introduced stronger conditions, such as the 
$\beta$-mixing condition.  Specifically, 
\cite{yu:94,karandikar:02} have proven asymptotic rates of uniform convergence 
for VC classes under stationary $\beta$-mixing processes.
One implication of this result is an
asymptotic rate of convergence for the excess risk of empirical risk minimization.
Other works have established rates of convergence for the excess risk
of empirical risk minimization and other learning methods, under related mixing conditions, 
including $\alpha$-mixing \cite{vidyasagar:03}, $\eta$-mixing \cite{kontorovich:07}, 
and $\phi$-mixing \cite{vidyasagar:03}, all under the stationarity assumption.

The other primary direction in the study of the risk of learning methods 
under relaxations of the i.i.d. assumption preserves the independence 
assumption, while allowing the marginal distributions to \emph{drift}
over time.  This thread in the literature has focused on the specific
setting of binary classification.  Specifically, 
\cite{long:99,helmbold:91,helmbold:94,barve:96,barve:97,crammer:10} 
study a setting in which the marginal distribution of the data point at time $t$ has total variation
distance from that of the data point at time $t+1$ at most a given 
upper bound, called the \emph{drift rate} (see also related work by \cite{Bartlett:1992:LSC:130385.130412,freund:97,bartlett:00,drifting-distribution,mohri:12}).
The data points are still assumed to be independent.
The recent works of \cite{concept-drift,mohri:12}
further explore this problem (in a formulation more-closely paralleling
that studied here).  In this setting, the learning method produces a sequence
of predictors (e.g., classifiers), where the method for choosing the 
predictor at time $t$ may depend on all of the data up to time $t-1$.
The results in these works are expressible as bounds on the 
risk at each time $t$ (or sometimes averaged over time), 
as a function of $t$ and the rates of drift of the marginal distributions.

The paper of \cite{mohri:12} also studies a refinement of the notion of ``drift''
compared to the earlier works, such as \cite{barve:96,barve:97}.  Specifically, 
rather than measuring the difference between the next and previous distributions 
by the total variation distance, they instead use a notion of ``discrepancy'' that 
depends directly on the function class being used for learning.  This discrepancy is 
sometimes significantly smaller than the total variation distance, yet plays an 
analogous role in the bounds of \cite{mohri:12} as the total variation distance plays 
in the bounds of \cite{helmbold:94,barve:97}.
To allow for this refined notion of drift, our arguments below are phrased generally 
enough that they can be applied with either notion of drift (discrepancy or total variation).

In recent work, \cite{kuznetsov:14} discusses the problem of learning 
from non-stationary mixing processes.  They derive interesting results bounding
the risk at some future time in terms of the empirical risk on all observed data,
with clear implications for the performance of methods such as empirical risk
minimization.  The nature of the results in that work are somewhat different from 
our results below.  However, the spirit of the analysis is similar in many places,
and one can conceivably convert some of those results into a more-closely related
form with a bit of additional effort.  

One significant point of divergence between the present work and that of 
\cite{kuznetsov:14}, and indeed all of the above works on product processes 
(aside from certain special cases discussed by \cite{concept-drift}),
is that in the general case, these works require access to the sequence of magnitudes 
of drift of the distribution, or a constant upper bound thereon.  The sequence of 
drift magnitudes is a substantial number of variables to assume we have access to
(linear in the number of data points),
and relying only on a constant upper bound precludes the possibility of sublinear
growth of the cumulative excess risk \cite{helmbold:94,concept-drift}.  The notion of discrepancy
studied by \cite{mohri:12,kuznetsov:14} (see below) can sometimes be estimated from data, 
but only under significant further restrictions on the process.
In contrast, in the present work, we merely assume an asymptotic bound on the rate
of growth of the cumulative amount of drift.  Our learning method then depends only
on the single parameter that this asymptotic growth rate is described in terms of,
and we show that this is enough to achieve sublinear growth of the cumulative 
excess risk, without needing access to the sequence of drift rates or additional restrictions
on the process.  For completeness, we also briefly discuss the case where the
drift rates are known, in Section~\ref{sec:product}.

The present work studies learning under general nonstationary processes,
under a condition that allows us to extend the ideas from the above-described
literature on learning from product processes with slowly-drifting marginal 
distributions.  Specifically, we replace the independence condition with a $\beta$-mixing
condition.
In addition to this, we suppose that the sum of distances between 
marginal distributions at adjacent time steps grows only sublinearly
(note that this does \emph{not} require that the sequence of distributions 
be converging).
Our objective is then to propose a prediction strategy (for producing 
the $f_{t}$ function), and to characterize the rate of growth of the 
cumulative excess risk over time.
The excess risks are calculated relative to the sequence of \textit{a priori} optimal predictors
among functions in a given function class.
In particular, for any bounded VC subgraph class, we establish 
a rate of growth of the cumulative excess risk that is \emph{sublinear} in the 
number of predictions made.

\subsection{Definitions and Summary of Main Result}
To formalize this setting, we adopt the abstract perspective of the 
\emph{general learning setting} of \cite{vapnik:82,vapnik:98}.  Specifically,
fix a measurable space $(\Z,\BorelZ)$ and a \emph{function class} $\F$ of measurable functions 
$f : \Z \to [0,1]$.  
For instance, in the special case of classification, $\Z$ would be a set of $(x,y)$ pairs, 
and $\F$ would be a set of functions $f_{h}((x,y)) = \ind[h(x) \neq y]$, where $h$ ranges over a 
set $\H$ of functions (known as the hypothesis class); see \cite{koltchinskii:06,shalev-shwartz:10} for many other examples.
In the general learning setting, the aim of a learning algorithm is to identify a function $f \in \F$
with a relatively small average value, where the average is taken with respect to some unknown
probability measure on $\Z$ (as discussed in more detail below).  For instance, in the classification 
setting described above, this average value corresponds to the probability that $h$ makes a ``mistake'' 
in predicting the value of $y$ from $x$.

For simplicity, to avoid the common measurability issues arising in 
empirical process theory, we will suppose $\F$ is such that the events involved in the proofs below are all measurable
(for instance, this is certainly the case if $\F$ is countable; see \cite{van-der-Vaart:96} for other sufficient conditions).
%
Let $\vc$ denote the pseudo-dimension of $\F$ \cite{pollard:84,pollard:90,haussler:92,anthony:99}: 
that is, $\vc$ is the largest $k \in \nats \cup \{0\}$ such that $\exists (z_{1},w_{1}),\ldots,(z_{k},w_{k}) \in \Z \times \reals$ with
$|\{ ( \ind[ f(z_{1}) \leq w_{1} ], \ldots, \ind[ f(z_{k}) \leq w_{k} ] ) : f \in \F \}| = 2^{k}$, or is $\infty$ if no such largest $k$ exists.
Throughout this article, we suppose $1 \leq \vc < \infty$ (so that $\F$ is a VC Subgraph class). 

We suppose there is a sequence of $\Z$-valued random variables $Z_{1},Z_{2},\ldots$, called the \emph{data points},
and for each $t \in \nats$, we denote by $P_{t}$ the marginal distribution of the random variable $Z_{t}$.  Also, 
generally, for any random variable $X$, we denote by $\P_{X}$ the distribution of $X$ (i.e., $\P_{X}(\cdot) = \P(X^{-1}(\cdot))$).
For any probability measures $P,Q$ on a measurable space $(\Omega,\Borel)$, we denote by $\|P-Q\| = \sup_{A \in \Borel} P(A)-Q(A)$
the total variation distance between $P$ and $Q$.  Additionally, for probability measures $P,Q$ on the measurable space $(\Z,\BorelZ)$,
we denote by 
\begin{equation*}
\disc(P,Q) = \sup_{f \in \F} \left| \E_{Z \sim P}[f(Z)] - \E_{Z \sim Q}[f(Z)] \right|,
\end{equation*}
a general notion of \emph{discrepancy} introduced by \cite{mansour:09,mohri:12}.
We use $\disc$ below to quantify the magnitude of change in the marginal distribution
of $Z_{t+1}$ compared to $Z_{t}$.
Note that, since every $f \in \F$ is uniformly bounded in $[0,1]$, we clearly have 
\begin{equation*}
\disc(P,Q) \leq \|P-Q\|.
\end{equation*}
Indeed, readers more comfortable with the familiar total variation distance may feel free to replace $\disc(P,Q)$
with $\|P-Q\|$ in all contexts below, and the results and proofs will remain valid without any further modifications.
However, one can construct scenarios in which $\disc(P,Q)$ provides a much smaller value, and generally $\disc(P,Q)$
appears to be more relevant to the learning setting than is the total variation distance.
For each $t \geq 2$, let $\drift_{t} \in [0,1]$ be a value
satisfying 
\begin{equation}
\label{eqn:drift-defn}
\disc(P_{t},P_{t-1}) \leq \drift_{t}.
\end{equation}
For completeness, also define $\drift_{1} = 0$.

To obtain nontrivial results, we are interested in restricting the family of processes.
Specifically, for our main result below (Theorem~\ref{thm:main}), we suppose 
\begin{equation}
\label{eqn:drift-rate}
\sum_{t=1}^{T} \drift_{t} = O(T^{\alpha}),
\end{equation}
for a given value $\alpha \in [0,1)$.
Note that this does not require that the sequence of distributions be converging,
only that its average rate of change slows over time.
We additionally adopt the standard definition of \emph{$\beta$-mixing}, defined as follows.
Following \cite{bradley:83} and \cite{yu:94},
for each $k \in \nats$, 
define
\begin{equation*}
\beta_{k} = 
\frac{1}{2} \sup\left\{ \sum_{i=1}^{I} \sum_{j=1}^{J} | \P(A_{i} \cap B_{j}) - \P(A_{i})\P(B_{j}) |
: \{A_{i}\}_{i} \in \Pi_{\ell}, \{B_{j}\}_{j} \in \Pi^{\prime}_{\ell+k}, \ell \geq 1\right\},
\end{equation*}
where $\Pi_{\ell}$ is defined as the set of $\sigma(\{Z_{1},\ldots,Z_{\ell}\})$-measurable finite partitions,
and $\Pi^{\prime}_{\ell+k}$ is defined as the set of 
$\sigma(\{Z_{\ell+k},Z_{\ell+k+1},\ldots\})$-measurable 
finite partitions.
Then we suppose
\begin{equation}
\label{eqn:beta-rate}
\beta_{k} = O(k^{-r}),
\end{equation}
for some $r \in (0,\infty)$.

Under the assumptions \eqref{eqn:drift-rate} and \eqref{eqn:beta-rate}, we propose a learning method, specified as follows.
Let $\hat{f}_{1}$ be arbitrary.
For each $t \in \nats \setminus \{1\}$, let 
\begin{equation*}
m_{t} = \left\lceil (t-1)^{(1-\alpha)\frac{3+2r}{3+3r}} \right\rceil
\end{equation*}
and 
\begin{equation*}
k_{t} = \left\lceil (t-1)^{(1-\alpha)\frac{1}{1+r}} \right\rceil 
\end{equation*}
and choose as a predictor at time $t$ a function\footnote{For simplicity, 
we suppose the minimum is actually \emph{achieved} by some $f \in \F$.  To handle the general case, all of the results continue to hold, 
with only minor technical changes to the proofs, if we instead choose $\hat{f}_{t} \in \F$ with 
$\sum_{s=1}^{\lfloor m_{t}/k_{t} \rfloor} \hat{f}_{t}(Z_{t - s k_{t}})$
sufficiently close to $\inf_{f \in \F}\sum_{s=1}^{\lfloor m_{t}/k_{t} \rfloor} f(Z_{t - s k_{t}})$.}
\begin{equation}
\label{eqn:hatf-defn}
\hat{f}_{t} = \argmin\limits_{f \in \F} \sum_{s=1}^{\lfloor m_{t}/k_{t} \rfloor} f(Z_{t - s k_{t}}).
\end{equation}
For $\hat{f}_{t}$ chosen in this way, we prove the following theorem.



\begin{theorem}
\label{thm:main}
If \eqref{eqn:drift-rate} and \eqref{eqn:beta-rate} are satisfied, then 
\begin{equation*}
\sum_{t=1}^{T} \E\left[ \hat{f}_{t}(Z_{t}) \right] - \sum_{t=1}^{T} \inf_{f \in \F} \E\left[ f(Z_{t}) \right]
= O\left( T^{\frac{3+(2+\alpha)r}{3+3r}} \right).
\end{equation*}
\end{theorem}

In particular, note that the expression on the right hand side grows \emph{sublinearly} in $T$.
To prove this theorem, we first provide two key lemmas from the literature,
after which we present the proof of Theorem~\ref{thm:main} below.
Following this, in Section~\ref{sec:product}, 
we conclude the paper by establishing \emph{finite-sample} bounds,
and other specialized results, in the special case of \emph{product} processes; this 
effectively extends to the general learning setting results established by \cite{barve:96,barve:97}
for binary classification, while also expressing the results in a more general 
form that allows for a time-varying drift rate.

\section{Proof of Theorem~\ref{thm:main}}
\label{sec:main-proof}

The following lemma is a well-known result on $\beta$-mixing processes, from \cite{volkonskii:59,eberlein:84} 
(see also Theorem 2.1 of \cite{vidyasagar:03} or Corollary 2.7 of \cite{yu:94}).
\begin{lemma}
\label{lem:yu-dep-to-indep}
For any $t,n,k \in \nats$,
\begin{equation*}
\left\| \P_{\{Z_{(j-1)k+t}\}_{j=1}^{n}} - \left( \times_{j=1}^{n} P_{(j-1)k+t} \right) \right\| \leq (n-1) \beta_{k}.
\end{equation*}
\end{lemma}

Additionally, we use the following well-known result 
(see e.g., \cite{van-der-Vaart:96}, Theorems 2.14.1 and 2.6.7). 

\begin{lemma}
\label{lem:uniform-convergence}
There exists a universal constant $c \in [1,\infty)$ such that, 
for any independent $\Z$-valued random variables $Z_{1}^{\prime},\ldots,Z_{m}^{\prime},$
\begin{equation*}
\E\left[ \sup_{f \in \F} \left| \frac{1}{m} \sum_{t=1}^{m} \left(f(Z_{t}^{\prime}) - \E[f(Z_{t}^{\prime})]\right) \right| \right]
\leq c \sqrt{\frac{\vc}{m}}.
\end{equation*}
\end{lemma}

While the 
proof of this result 
in \cite{van-der-Vaart:96} 
discusses only i.i.d. random variables,
the proof in fact implies this result, which only assumes independence.  For completeness, we 
include a brief proof in 
Appendix~\ref{app:uniform-convergence}.

With these lemmas in hand, we are ready to present the proof of Theorem~\ref{thm:main}.

\begin{proof}[Proof of Theorem~\ref{thm:main}]
Let $Z_{1}^{\prime},Z_{2}^{\prime},\ldots$ denote a sequence of independent random variables,
also independent from $\{Z_{i}\}_{i\in\nats}$,
and with each $Z_{i}^{\prime} \sim P_{i}$.
Fix any $t \in \nats \setminus \{1\}$.
Since $\hat{f}_{t}$ depends only on $Z_{1},\ldots,Z_{t-k_{t}}$, 
it follows immediately from the definition of $\beta_{k_{t}}$ (see \cite{yu:94}, Lemma 2.6) that
\begin{equation*}
\left\| \P_{(\hat{f}_{t},Z_{t})} - \P_{(\hat{f}_{t},Z_{t}^{\prime})} \right\|
= \left\| \P_{(\hat{f}_{t},Z_{t})} - \P_{\hat{f}_{t}} \times \P_{Z_{t}} \right\|
\leq \beta_{k_{t}}.
\end{equation*}
In particular, this implies
\begin{equation*}
\E\left[ \hat{f}_{t}(Z_{t}) \right]
\leq \E\left[ \hat{f}_{t}(Z_{t}^{\prime}) \right] + \beta_{k_{t}}.
\end{equation*}

Additionally, since 
$\disc(P_{t-i k_{t}},P_{t}) \leq \sum_{q=t-i k_{t}}^{t-1} \change_{q+1}$ 
for $1 \leq i \leq \lfloor m_{t}/k_{t} \rfloor$, and every $Z_{j}^{\prime}$ is independent of $\hat{f}_{t}$, we have that
\begin{align*}
\E\left[ \hat{f}_{t}(Z_{t}^{\prime}) \right]
& = \E\left[ \E\left[ \hat{f}_{t}(Z_{t}^{\prime}) \middle| \hat{f}_{t} \right] \right]
\\ & \leq \E\left[ \frac{1}{\lfloor m_{t}/k_{t} \rfloor} \sum_{i=1}^{\lfloor m_{t}/k_{t} \rfloor} \E\left[ \hat{f}_{t}(Z_{t-i k_{t}}^{\prime}) \middle| \hat{f}_{t} \right] \right] 
+ \frac{1}{\lfloor m_{t}/k_{t} \rfloor} \sum_{i=1}^{\lfloor m_{t}/k_{t} \rfloor} \sum_{q=t - i k_{t}}^{t-1} \change_{q+1}.
\end{align*}
Furthermore, 
\begin{align}
& \E\left[ \frac{1}{\lfloor m_{t}/k_{t} \rfloor} \sum_{i=1}^{\lfloor m_{t}/k_{t} \rfloor} \E\left[ \hat{f}_{t}(Z_{t-i k_{t}}^{\prime}) \middle| \hat{f}_{t} \right] \right] \notag
\\ & \leq 
\E\!\left[ \frac{1}{\lfloor m_{t}/k_{t} \rfloor} \sum_{i=1}^{\lfloor m_{t}/k_{t} \rfloor} \hat{f}_{t}(Z_{t-i k_{t}}) \right]
+ \E\!\left[ \sup_{f \in \F} \left| \frac{1}{\lfloor m_{t}/k_{t} \rfloor} \sum_{i=1}^{\lfloor m_{t}/k_{t} \rfloor} \left(\E[f(Z_{t-i k_{t}}^{\prime})] - f(Z_{t-i k_{t}}) \right) \right| \right]. \label{eqn:two-terms}
\end{align}
Now let us bound each term in \eqref{eqn:two-terms} separately.
First, we have that
\begin{align*}
& \E\left[ \frac{1}{\lfloor m_{t}/k_{t} \rfloor} \sum_{i=1}^{\lfloor m_{t}/k_{t} \rfloor} \hat{f}_{t}(Z_{t-i k_{t}}) \right]
= \E\left[ \inf_{f \in \F} \frac{1}{\lfloor m_{t}/k_{t} \rfloor} \sum_{i=1}^{\lfloor m_{t}/k_{t} \rfloor} f(Z_{t-i k_{t}}) \right]
\\ & \leq \inf_{f \in \F} \frac{1}{\lfloor m_{t}/k_{t} \rfloor} \sum_{i=1}^{\lfloor m_{t}/k_{t} \rfloor} \E\left[ f(Z_{t-i k_{t}}) \right]
\leq \inf_{f \in \F} \E[f(Z_{t})] + \frac{1}{\lfloor m_{t}/k_{t} \rfloor} \sum_{i=1}^{\lfloor m_{t}/k_{t} \rfloor} \sum_{q=t - i k_{t}}^{t-1} \change_{q+1}.
\end{align*}
Next, Lemma~\ref{lem:yu-dep-to-indep} implies
\begin{align*}
& \E\left[ \sup_{f \in \F} \left| \frac{1}{\lfloor m_{t}/k_{t} \rfloor} \sum_{i=1}^{\lfloor m_{t}/k_{t} \rfloor} \left(\E[f(Z_{t-i k_{t}}^{\prime})] - f(Z_{t-i k_{t}}) \right) \right| \right]
\\ & \leq \E\left[ \sup_{f \in \F} \left| \frac{1}{\lfloor m_{t}/k_{t} \rfloor} \sum_{i=1}^{\lfloor m_{t}/k_{t} \rfloor} \left(\E[f(Z_{t-i k_{t}}^{\prime})] - f(Z_{t-i k_{t}}^{\prime}) \right) \right| \right] + \left( \lfloor m_{t}/k_{t} \rfloor - 1 \right) \beta_{k_{t}}.
\end{align*}
Furthermore, Lemma~\ref{lem:uniform-convergence} implies
\begin{equation*}
\E\left[ \sup_{f \in \F} \left| \frac{1}{\lfloor m_{t}/k_{t} \rfloor} \sum_{i=1}^{\lfloor m_{t}/k_{t} \rfloor} \left(\E[f(Z_{t-i k_{t}}^{\prime})] - f(Z_{t-i k_{t}}^{\prime}) \right) \right| \right]
\leq c \sqrt{\frac{\vc}{\lfloor m_{t}/k_{t} \rfloor}}. 
\end{equation*}
Together, we have that \eqref{eqn:two-terms} is at most
\begin{equation*}
\inf_{f \in \F} \E[f(Z_{t})] + \left(\frac{1}{\lfloor m_{t}/k_{t} \rfloor} \sum_{i=1}^{\lfloor m_{t}/k_{t} \rfloor} \sum_{q=t - i k_{t}}^{t-1} \change_{q+1}\right)
+ c \sqrt{\frac{\vc}{\lfloor m_{t}/k_{t} \rfloor}}
+ \left(\lfloor m_{t}/k_{t} \rfloor - 1\right)\beta_{k_{t}}.
\end{equation*}

Altogether, 
we have established that
\begin{equation}
\label{eqn:the-bound-term}
\E\left[ \hat{f}_{t}(Z_{t}) \right]
\leq \inf_{f \in \F} \E[f(Z_{t})] 
+ 2 \left(\frac{1}{\lfloor m_{t}/k_{t} \rfloor} \sum_{i=1}^{\lfloor m_{t}/k_{t} \rfloor} \sum_{q=t - i k_{t}}^{t-1} \change_{q+1}\right)
+ c \sqrt{\frac{\vc}{\lfloor m_{t}/k_{t} \rfloor}}
+ \lfloor m_{t}/k_{t} \rfloor \beta_{k_{t}}.
\end{equation}
Therefore, 
\begin{multline}
\label{eqn:the-bound}
\sum_{t=1}^{T} \E\left[ \hat{f}_{t}(Z_{t}) \right] - \sum_{t=1}^{T} \inf_{f \in \F} \E\left[ f(Z_{t}) \right] 
\\ \leq 1 + \left(\sum_{t=2}^{T} \frac{2}{\lfloor m_{t}/k_{t} \rfloor} \sum_{i=1}^{\lfloor m_{t}/k_{t} \rfloor} \sum_{q=t - i k_{t}}^{t-1} \change_{q+1} \right)
+ \left(\sum_{t=2}^{T} c \sqrt{\frac{\vc}{\lfloor m_{t}/k_{t} \rfloor}}\right)
+ \left(\sum_{t=2}^{T} \lfloor m_{t}/k_{t} \rfloor \beta_{k_{t}}\right).
\end{multline}

All that remains is to bound each of these three terms on the right hand side of \eqref{eqn:the-bound}.
The only term presenting a challenge in this regard is the term involving the $\change_{q+1}$ values, 
and for that reason we leave this term for last.
For the other terms, first note that
\begin{equation*}
\sum_{t=1}^{T} t^{-(1-\alpha)\frac{r}{3+3r}}
= O\left( 1 + \int_{1}^{T} t^{-(1-\alpha)\frac{r}{3+3r}} {\rm d}t \right)
= O\left( T^{\frac{3+(2+\alpha)r}{3+3r}} \right).
\end{equation*}
Thus, we have that
\begin{equation}
\label{eqn:the-bound-term-1}
\sum_{t=2}^{T} c \sqrt{\frac{\vc}{\lfloor m_{t}/k_{t} \rfloor}}
= O\left( \sum_{t=1}^{T} t^{-(1-\alpha)\frac{r}{3+3r}} \right)
= O\left( T^{\frac{3+(2+\alpha)r}{3+3r}} \right).
\end{equation}
Also, we have
\begin{equation}
\label{eqn:the-bound-term-2}
\sum_{t=2}^{T} \lfloor m_{t}/k_{t} \rfloor \beta_{k_{t}}
= O\left( \sum_{t=2}^{T} m_{t} / k_{t}^{1+r} \right)
= O\left( \sum_{t=1}^{T} t^{-(1-\alpha)\frac{r}{3+3r}} \right)
= O\left( T^{\frac{3+(2+\alpha)r}{3+3r}} \right).
\end{equation}

The remaining term, $\sum_{t=2}^{T} \frac{2}{\lfloor m_{t}/k_{t} \rfloor} \sum_{i=1}^{\lfloor m_{t}/k_{t} \rfloor} \sum_{q=t - i k_{t}}^{t-1} \change_{q+1}$, 
requires more work to bound.
First note that 
\begin{equation*}
\sum_{t=2}^{T} \frac{2}{\lfloor m_{t}/k_{t} \rfloor} \sum_{i=1}^{\lfloor m_{t}/k_{t} \rfloor} \sum_{q=t - i k_{t}}^{t-1} \change_{q+1}
\leq 2 \sum_{t=2}^{T} \sum_{q = t - m_{t}}^{t-1} \change_{q+1}.
\end{equation*}
We will focus on bounding the right hand side.
Now note that every value of $t \in \nats$ for which $q \in \{t-m_{t},\ldots,t-1\}$ 
satisfies 
\begin{align*}
2q 
\geq 2t - 2m_{t} 
& = 2t \left( 1 - \frac{m_{t}}{t} \right) 
\geq 2t \left( 1 - 2 (t-1)^{(1-\alpha)\frac{3+2r}{3+3r} - 1} \right) 
\\ & = 2t \left( 1 - 2 (t-1)^{ - \frac{3\alpha + 2r\alpha + r}{3+3r}} \right)
\geq 2t \left( 1 - 2 q^{ - \frac{3\alpha + 2r\alpha + r}{3+3r}} \right)
\geq 2t \left( 1 - 2 q^{ - \frac{r}{3+3r}} \right).
\end{align*}
Denote $q_{r} = \left\lceil 4^{\frac{3+3r}{r}} \right\rceil$, 
and note that for any $q \geq q_{r}$ we have $2t \left( 1 - 2 q^{ - \frac{r}{3+3r}} \right) \geq t$.
Thus, for any $q \geq q_{r}$, every $t \in \nats$ with $q \in \{t-m_{t},\ldots,t-1\}$ has $t \leq 2q$, 
so that (by monotonicity of $m_{t}$) we also have $q \in \{t - m_{2q}, \ldots, t-1\}$, 
or equivalently $t \in \{q+1, \ldots, q+m_{2q}\}$.  In particular, this means any such $q$ has at most $m_{2q}$ 
appearances of the quantity $\change_{q+1}$ in the summation $\sum_{t=2}^{T} \sum_{q=t-m_{t}}^{t-1} \change_{q+1}$.
Also, clearly the largest $q$ with $\change_{q+1}$ appearing in this summation is $q=T-1$.
Additionally, since $m_{t}$ is sublinear in $t$, we have $t-m_{t} \to \infty$ as $t \to \infty$, 
so that there is some finite $t_{0}$ such that every $t > t_{0}$ has $t-m_{t} \geq q_{r}$.
Thus, every $q < q_{r}$ has $\change_{q+1}$ appearing at most $t_{0}$ times in the summation 
$\sum_{t=2}^{T} \sum_{q=t-m_{t}}^{t-1} \change_{q+1}$.
Altogether, we have that 
\begin{align*}
2 \sum_{t=2}^{T} \sum_{q=t-m_{t}}^{t-1} \change_{q+1}
& \leq 2 t_{0} \sum_{q=1}^{q_{r} - 1} \change_{q+1} + 2 \sum_{q=q_{r}}^{T-1} m_{2q} \change_{q+1}
\\ & = O\!\left( m_{2T} \sum_{q=1}^{T} \change_{q} \right)
= O\!\left( T^{(1-\alpha) \frac{3+2r}{3+3r} + \alpha} \right)
= O\!\left( T^{\frac{3+(2+\alpha)r}{3+3r}} \right), 
\end{align*}
where we have used the assumption \eqref{eqn:drift-rate} on the $\change_{t}$ sequence.

Plugging this bound into \eqref{eqn:the-bound} along with \eqref{eqn:the-bound-term-1} and \eqref{eqn:the-bound-term-2},
we have established that
\begin{equation*}
\sum_{t=1}^{T} \E\left[ \hat{f}_{t}(Z_{t}) \right] - \sum_{t=1}^{T} \inf_{f \in \F} \E\left[ f(Z_{t}) \right] 
= O\!\left( T^{\frac{3+(2+\alpha)r}{3+3r}} \right),
\end{equation*}
which completes the proof.
\end{proof}

\section{Product Processes}
\label{sec:product}

In this section, unlike above, we suppose the algorithm has direct access to the $\change_{t}$ sequence.
Our objective is then to derive more-explicit (non-asymptotic) bounds under the assumption that $\{Z_{t}\}_{t=1}^{\infty}$ is a product process.
The results here are already known in the special case of binary classification, 
in the case that $\change_{t}$ is bounded by a $t$-invariant \emph{constant} for all $t$ \cite{barve:97}.
Thus, this section represents a generalization of these classic results to the general learning setting, 
and to general time-varying drift rates.  That said, we note that the results here would also readily follow
from the classic analysis of \cite{barve:97} and the more-recent work of \cite{mohri:12}, with only minor
additional work to apply those results to a recent history of data points trailing the prediction time $t$; 
there is nevertheless some value in stating the results explicitly here, particularly since they follow directly 
from our analysis above.

Throughout this section, 
for any functions $f,g : A \to [0,\infty)$, for any set $A$, 
we write $f(a) \lesssim g(a)$ to express the claim that there exists a numerical constant $c \in (0,\infty)$
such that $f(a) \leq c g(a)$ for all $a \in A$; this allows us to express non-asymptotic bounds (in terms of $T$, $\vc$, and the $\change_{t}$ sequence), 
without concerning ourselves with precise numerical constant factors.
%
For each $t \in \nats \setminus \{1\}$, define
\[
\tilde{m}_{t} = \argmin_{m \in \{1,\ldots,t-1\}} \left( \sum_{q=t-m}^{t-1} \change_{q+1} + \sqrt{\frac{\vc}{m}} \right)
\]
and
\[
\tilde{f}_{t} = \argmin_{f \in \F} \sum_{s=t-\tilde{m}_{t}}^{t-1} f(Z_{s}).
\]
For completeness, define $\tilde{f}_{1}$ as an arbitrary element of $\F$.

\begin{theorem}
\label{thm:product-general}
If $\{Z_{t}\}_{t=1}^{\infty}$ is a product process, then for $T \in \nats \setminus \{1\}$, 
\begin{equation*}
\sum_{t=1}^{T} \E\left[ \tilde{f}_{t}(Z_{t}) \right] 
- \sum_{t=1}^{T} \inf_{f \in \F} \E[ f(Z_{t}) ] 
\lesssim \sum_{t=2}^{T} \min_{m \in \{1,\ldots,t-1\}} \left( \sum_{q=t-m}^{t-1} \change_{q+1} + \sqrt{\frac{\vc}{m}} \right).
\end{equation*}
\end{theorem}
\begin{proof}
We begin by noting that, in the proof of Theorem~\ref{thm:main}, 
the argument leading to \eqref{eqn:the-bound} in fact more generally
holds for \emph{any} $\beta$-mixing process $\{Z_{t}\}_{t \in \nats}$ (regardless of whether 
\eqref{eqn:drift-rate} and \eqref{eqn:beta-rate} are satisfied for the corresponding $\change_{t}$ and $\beta_{k}$ sequences),
and for \emph{any} sequence $\hat{f}_{t}$ defined as in \eqref{eqn:hatf-defn},
where the values $m_{t},k_{t} \in \nats$ can be specified \emph{arbitrarily},
subject to $k_{t} \leq m_{t} \leq t-1$.  In particular, substituting $k_{t} = 1$ and $m_{t} = \tilde{m}_{t}$,
the corresponding $\hat{f}_{t}$ from \eqref{eqn:hatf-defn} is precisely $\tilde{f}_{t}$.
Then since $\beta_{1} = 0$ for product processes, 
\eqref{eqn:the-bound} implies 
\begin{align*}
\sum_{t=1}^{T} \E\left[ \tilde{f}_{t}(Z_{t}) \right] - \sum_{t=1}^{T} \inf_{f \in \F} \E\left[ f(Z_{t}) \right]
& \lesssim \sum_{t=2}^{T} \left( \sum_{q=t-\tilde{m}_{t}}^{t-1} \change_{q+1} + \sqrt{\frac{\vc}{\tilde{m}_{t}}} \right)
\\ & = \sum_{t=2}^{T} \min_{m \in \{1,\ldots,t-1\}} \left( \sum_{q=t-m}^{t-1} \change_{q+1} + \sqrt{\frac{\vc}{m}} \right).
\end{align*}
\end{proof}

It remains an interesting open problem to determine whether the above guarantee is 
achievable by a learning rule that has no direct dependence on the $\change_{t}$ values:
that is, a method that is \emph{adaptive} to variations in the rates of drift.  Resolution
of this question seems an important step toward applicability of these ideas in practice.
Of course, as established in Theorem~\ref{thm:main}, if we instead assume that the asymptotic
bound \eqref{eqn:drift-rate} holds, then it is possible to replace the direct dependence on $\change_{t}$
with a mere dependence on a single parameter $\alpha$; however, the price for this is that
the finite-sample bound in Theorem~\ref{thm:product-general} would be replaced by an 
asymptotic guarantee.  An alternative option is to suppose the drift rates $\change_{t}$ are \emph{bounded}
by a value $\gamma$,
and then provide an algorithm depending only on $\gamma$; this coarse condition on $\change_{t}$
precludes the possibility of a sublinear cumulative excess risk guarantee, but it can nonetheless be interesting to study 
the dependence of the achieved excess risk on $\gamma$.  This is the subject of the next subsection.

\subsection{Constant Drift Rate}

In the context of binary classification, \cite{long:99,helmbold:91,helmbold:94,barve:96,barve:97,crammer:10,concept-drift} 
have derived bounds on the sequence of risks (or the number of mistakes) 
achieved by various methods, under the assumptions that $\{Z_{t}\}_{t=1}^{\infty}$
is a product process, and that $\change_{t} \leq \gamma$, for some fixed constant $\gamma \in (0,1)$.
Here we briefly note that some of these results (and in particular, those of \cite{barve:97}) can 
be generalized to the general learning setting, where we find analogous results on the average 
of the $\hat{f}_{t}(Z_{t})$ function values.
We note that a similar type of result can also be extracted from the analysis of \cite{mohri:12} 
with minor additional work to convert to our sequential setting.
%
%
%

Let $\bar{m} = \left\lceil \vc^{1/3} \gamma^{-2/3} \right\rceil$.
For each integer $t > \bar{m}$, let 
\begin{equation*}
\bar{f}_{t} = \argmin_{f \in \F} \sum_{s=t-\bar{m}}^{t-1} f(Z_{s}).
\end{equation*}
For completeness, for $t \leq \bar{m}$ define $\bar{f}_{t}$ as an arbitrary element of $\F$.

\begin{theorem}
\label{thm:product-bounded}
If $\{Z_{t}\}_{t=1}^{\infty}$ is a product process, then for $T > 1/\gamma$, 
\begin{equation*}
\sum_{t=1}^{T} \E\left[ \bar{f}_{t}(Z_{t}) \right] - \sum_{t=1}^{T} \inf_{f \in \F} \E[ f(Z_{t}) ]
\lesssim \left( \vc \gamma \right)^{1/3} T.
\end{equation*}
\end{theorem}

It is worth noting that the bound in Theorem~\ref{thm:product-bounded} would also 
hold for the predictor $\tilde{f}_{t}$ from Theorem~\ref{thm:product-general}; indeed,
this follows immediately from plugging in $\gamma$ for the values of $\change_{t}$,
in which case $\tilde{f}_{t}$ itself is quite similar to $\bar{f}_{t}$. 
%
However, as $\bar{f}_{t}$ admits the above simplified explicit form in this special case, 
we include a brief direct proof of this result as follows. 

\begin{proof}
As in the proof of Theorem~\ref{thm:product-general}, 
the proof is based on the general validity of \eqref{eqn:the-bound-term}.
In particular, taking $k_{t} = 1$ and $m_{t} = \min\!\left\{\bar{m}, (t-1)\right\}$,
the corresponding $\hat{f}_{t}$ is equal $\bar{f}_{t}$ for all $t > \bar{m}$.
Thus, \eqref{eqn:the-bound-term} implies
\begin{align*}
\sum_{t=1}^{T} \E\left[ \bar{f}_{t}(Z_{t}) \right] - \sum_{t=1}^{T} \inf_{f \in \F} \E\left[ f(Z_{t}) \right]
& \lesssim \bar{m} + \sum_{t=\bar{m}+1}^{T} \left( \sum_{q=t-\bar{m}}^{t-1} \change_{q+1} + \sqrt{\frac{\vc}{\bar{m}}} \right)
\\ & \leq \bar{m} + \sum_{t=\bar{m}+1}^{T} \left( \bar{m}\gamma + \left(\vc\gamma\right)^{1/3} \right)
\lesssim \vc^{1/3}\gamma^{-2/3} + \left(\vc\gamma\right)^{1/3} T.
\end{align*}
The proof is completed by noting that, for $T > 1/\gamma$, 
we have $(\vc\gamma)^{1/3} T > \vc^{1/3} \gamma^{-2/3}$, so that
$\vc^{1/3}\gamma^{-2/3} + \left(\vc\gamma\right)^{1/3} T < 2 (\vc\gamma)^{1/3} T$.
\end{proof}

\section{Discussion and Open Problems}

%

There remains an interesting question of whether the rate established in Theorem~\ref{thm:main}
is optimal.  In the case of stationary $\beta$-mixing processes, 
the best known result is $O\!\left( T^{\frac{3+r}{3+2r}} \right)$ \citep*{karandikar:02}.
This result can be recovered with our technique by setting $m_{t} = t-1$ and $k_{t} = \left\lceil (t-1)^{\frac{3}{3+2r}} \right\rceil$, 
noting that the term in \eqref{eqn:the-bound} depending on the $\change_{t}$ values is equal $0$ in the stationary case; 
indeed, to achieve this rate we required only that $\change_{t}=0$ for all $t$, which is a strictly weaker requirement than stationarity.
Stationary processes are a special case of $\alpha=0$ in \eqref{eqn:drift-rate}.
However, the result given in Theorem~\ref{thm:main} for $\alpha=0$ obtains a somewhat faster growth of $O\!\left( T^{\frac{3+2r}{3+3r}} \right)$.
Since the general case of $\alpha=0$ includes many nonstationary processes as well, 
it is not clear whether Theorem~\ref{thm:main} can be improved to provide a rate $O\!\left(T^{\frac{3+r}{3+2r}}\right)$ 
for general processes having $\alpha=0$.
If so, it would seem to require a different approach to the analysis, since if we were to take 
$m_{t} = t-1$ and $k_{t} = \left\lceil (t-1)^{\frac{3}{3+2r}} \right\rceil$ for a general process with $\alpha=0$,
the summation involving the $\change_{t}$ sequence in \eqref{eqn:the-bound} might then potentially grow faster than 
$T^{\frac{3+r}{3+2r}}$.  
Complementary to this question is the problem of establishing lower bounds on the minimax rates, 
which seems to require development of novel techniques for constructing nonstationary mixing processes 
for which the learning problem is challenging.

\appendix

\section{Proof of Lemma~\ref{lem:uniform-convergence}}
\label{app:uniform-convergence}

Since technically the original proof of Lemma~\ref{lem:uniform-convergence} was stated for identically distributed samples,
for completeness we present a brief proof of the result without this restriction.
The details follow a standard argument.
Specifically, following the usual symmetrization argument (e.g., \cite{boucheron:13}, Lemma 11.4), for $(Z_{1}^{\prime\prime},\ldots,Z_{m}^{\prime\prime})$
an independent copy of $(Z_{1}^{\prime},\ldots,Z_{m}^{\prime})$, and $\epsilon_{1},\ldots,\epsilon_{m}$ i.i.d. ${\rm Uniform}(\{-1,+1\})$ independent of all $Z_{i}^{\prime}$ and $Z_{i}^{\prime\prime}$,
by Jensen's inequality we have
\begin{align*}
\E\left[ \sup_{f \in \F} \left| \frac{1}{m} \sum_{t=1}^{m} (f(Z_{t}^{\prime})-\E[f(Z_{t}^{\prime})]) \right| \right]
& = \E\left[ \sup_{f \in \F} \left| \frac{1}{m} \sum_{t=1}^{m} (f(Z_{t}^{\prime})-\E[f(Z_{t}^{\prime\prime})]) \right| \right]
\\ \leq \E\left[ \sup_{f \in \F} \left| \frac{1}{m} \sum_{t=1}^{m} (f(Z_{t}^{\prime})-f(Z_{t}^{\prime\prime})) \right| \right]
& = \E\left[ \sup_{f \in \F} \left| \frac{1}{m} \sum_{t=1}^{m} \epsilon_{t}(f(Z_{t}^{\prime})-f(Z_{t}^{\prime\prime})) \right| \right]
\\ & \leq 2 \E\left[ \sup_{f \in \F} \left| \frac{1}{m} \sum_{t=1}^{m} \epsilon_{t}f(Z_{t}^{\prime}) \right| \right].
\end{align*}
Then Lemma 6.1 of \cite{massart:07} implies 
\begin{equation*}
\E\left[ \sup_{f \in \F} \bigg| \sum_{t=1}^{m} \epsilon_{t} f(Z_{t}^{\prime})/\sqrt{m} \bigg| \middle| Z_{1}^{\prime},\ldots,Z_{m}^{\prime} \right]
\leq 3\sum_{j=0}^{\infty} 2^{-j} \sqrt{ \ln( \mathcal{M}(2^{-j-1},\F,L_{2}(P_{m}^{\prime}) ) )},
\end{equation*}
where $\mathcal{M}(\delta,\F,L_{p}(P_{m}^{\prime}))$ is the $\delta$-packing number of $\F$ under $L_{p}(P_{m}^{\prime})$,
and $P_{m}^{\prime}$ is the empirical measure induced by $Z_{1}^{\prime},\ldots,Z_{m}^{\prime}$.
Since functions in $\F$ are bounded in $[0,1]$, 
$\mathcal{M}(\delta,\F,L_{2}(P_{m}^{\prime})) \leq \mathcal{M}(\delta^{2},\F,L_{1}(P_{m}^{\prime}))$,
and Theorem 6 of \cite{haussler:92} (based on Lemma 25 of \cite{pollard:84}) implies 
$\mathcal{M}(\delta^{2},\F,L_{1}(P_{m}^{\prime})) \leq 2 \left(\frac{2e}{\delta^{2}}\ln\frac{2e}{\delta^{2}}\right)^{\vc}$.
Thus,
$\sum_{j=0}^{\infty} 2^{-j} \sqrt{ \ln( \mathcal{M}(2^{-j-1},\F,L_{2}(P_{m}^{\prime}) ) )}
\leq c^{\prime} \sqrt{\vc}$ for a numerical constant $c^{\prime}$.
Combining the above inequalities yields the result.

\subsubsection*{Acknowledgments}

We thank Tommi Jaakkola for several helpful discussions.

\bibliographystyle{natbib}
\bibliography{bib-concept-drift}

\ignore{
\providecommand{\bysame}{\leavevmode\hbox to3em{\hrulefill}\thinspace}
\providecommand{\MR}{\relax\ifhmode\unskip\space\fi MR }
\providecommand{\MRhref}[2]{%
  \href{http://www.ams.org/mathscinet-getitem?mr=#1}{#2}
}
\providecommand{\href}[2]{#2}

}

\end{document}